\documentclass{article}

\usepackage[preprint]{neurips_2024}

\usepackage[utf8]{inputenc} %
\usepackage[T1]{fontenc}    %
\usepackage{hyperref}       %
\usepackage{url}            %
\usepackage{booktabs}       %
\usepackage{amsfonts}       %
\usepackage{nicefrac}       %
\usepackage{microtype}      %
\usepackage{xcolor}         %

\usepackage{graphicx}
\usepackage{subcaption}
\usepackage{booktabs} %
\usepackage{float}
\usepackage[mathscr]{eucal}

\usepackage{amsmath}
\usepackage{amssymb}
\usepackage{mathtools}
\usepackage{amsthm}

\usepackage[capitalize,noabbrev]{cleveref}

\theoremstyle{plain}

\usepackage[inline]{enumitem}
\usepackage{multirow}

\usepackage{wrapfig, framed, caption}
\usepackage{multirow}
\hypersetup{colorlinks,linkcolor={blue},citecolor={blue},urlcolor={black}}  

\usepackage[textsize=tiny]{todonotes}

\usepackage{stmaryrd}
\pdfoutput=1

\newtheorem{prop}{Proposition}
\newtheorem{thm}{Theorem}

\newtheorem{lem}{Lemma}
\newtheorem{defn}{Definition}
\DeclareMathOperator{\Id}{Id}
\DeclareMathOperator{\mymin}{min}

  \newcommand*{\defeq}{\coloneqq}

\DeclareMathOperator*{\argmax}{arg\,max} 
\DeclareMathOperator*{\argmin}{arg\,min} 
 
\DeclareMathOperator{\prox}{prox}

\definecolor{mygreen}{rgb}{0.0, 0.7, 0.0}
\newcommand{\R}{\mathbb{R}}

\newcommand{\bff}{\mathbf{f}}
\newcommand{\ba}{\mathbf{a}}
\newcommand{\bx}{\mathbf{x}}
\newcommand{\bX}{\mathbf{X}}
\newcommand{\bY}{\mathbf{Y}}
\newcommand{\bg}{\mathbf{g}}
\newcommand{\bb}{\mathbf{b}}
\newcommand{\by}{\mathbf{y}}

\newcommand{\bz}{\mathbf{z}}
\newcommand{\bw}{\mathbf{w}}
\newcommand{\bu}{\mathbf{u}}

\newcommand{\bp}{\mathbf{p}}

\newcommand{\eps}{\varepsilon}

\usepackage{cancel}
\usepackage[normalem]{ulem}

\usepackage[capitalize,noabbrev]{cleveref}

\usepackage[textsize=tiny]{todonotes}

\include{math_commands}

\title{Learning Elastic Costs to Shape Monge Displacements}

\author{%
  Michal Klein\\
  Apple\\
  \texttt{michalk@apple.com}\\
  \And
  Aram-Alexandre Pooladian\\
  NYU\\
  \texttt{aram-alexandre.pooladian@nyu.edu} \\
  \And
  Pierre Ablin\\
  Apple\\
  \texttt{p\_ablin@apple.com}\\
  \AND
  Eugène Ndiaye\\
  Apple\\
  \texttt{e\_ndiaye@apple.com}\\
  \And
  Jonathan Niles-Weed\\
  NYU\\
  \texttt{jnw@cims.nyu.edu}\\
  \And
  Marco Cuturi\\
  Apple\\
  \texttt{cuturi@apple.com}\\
}

\begin{document}

\maketitle

\begin{abstract}
Given a source and a target probability measure supported on $\mathbb{R}^d$, the \citeauthor{Monge1781} problem aims for the most efficient way to map one distribution to the other.
This efficiency is quantified by defining a \textit{cost} function between source and target data. 
Such a cost is often set by default in the machine learning literature to the squared-Euclidean distance, $\ell^2_2(\bx,\by)=\tfrac12\|\bx-\by\|_2^2$.
The benefits of using \textit{elastic costs}, defined through a regularizer $\tau$ as $c(\bx,\by)=\ell^2_2(\bx,\by)+\tau(\bx-\by)$, was recently highlighted in~\citep{cuturi2023monge}. Such costs shape the \textit{displacements} of \citeauthor{Monge1781} maps $T$, i.e., the difference between a source point and its image $T(\bx)-\bx$, by giving them a structure that matches that of the proximal operator of $\tau$.
In this work, we make two important contributions to the study of elastic costs: \textit{(i)} For any elastic cost, we propose a numerical method to compute \citeauthor{Monge1781} maps that are provably optimal. This provides a much-needed routine to create synthetic problems where the ground truth OT map is known, by analogy to the \citeauthor{Bre91} theorem, which states that the gradient of any convex potential is always a valid \citeauthor{Monge1781} map for the $\ell_2^2$ cost; \textit{(ii)} We propose a loss to \textit{learn} the parameter $\theta$ of a parameterized regularizer $\tau_\theta$, and apply it in the case where $\tau_{A}(\bz)=\|A^\perp \bz\|^2_2$. This regularizer promotes displacements that lie on a low dimensional subspace of $\mathbb{R}^d$, spanned by the $p$ rows of $A\in\mathbb{R}^{p\times d}$. We illustrate the success of our procedure on synthetic data, generated using our first contribution, in which we show near-perfect recovery of $A$'s subspace using only samples. We demonstrate the applicability of this method by showing predictive improvements on single-cell data tasks.

\looseness=-1
\end{abstract}

\section{Introduction}
Finding efficient ways to map a distribution of points onto another is a low-level task that plays a crucial role across many machine learning (ML) problems.
Optimal transport (OT) theory~\citep{santambrogio2015optimal} has emerged as a tool of choice %
to solve such challenging matching problems in e.g.., scientific domains~\citep{schiebinger2019optimal,pmlr-v119-tong20a,bunne2023}.
We focus in this work on a crucial OT subtask, namely the numerical resolution of the \citeauthor{Monge1781} problem, which aims, on the basis of high-dimensional source and target data $(\bx_1,\dots,\bx_n)$ and $(\by_1,\dots,\by_m$), to recover a map $T:\mathbb{R}^d\rightarrow\mathbb{R}^d$ that is simultaneously \textit{(i)} a \emph{pushfoward} map, in the sense that $T$ applied on source samples recovers the distribution of target samples; \textit{(ii)} \emph{efficient}, in the sense that $T(\bx_i)$ is not too far, on average from $\bx_i$. The notion of efficiency can be made precise by parameterizing an OT problem with a real-valued cost function $c$ that compares a point $\bx$ and its mapping via $c(\bx,T(\bx))\in\mathbb{R}$. 
\paragraph{Challenges in the estimation of OT maps. }
Estimating OT maps is hindered, in principle, by the curse of dimensionality~\citep{hutter2021minimax}. This has led practitioners to seek OT solvers that leverage inductive biases. A simple workaround is to reduce the dimension of input data, using for instance a variational auto encoder~\citep{bunne2023}, or learning hyperplane projections jointly with OT estimation~\citep{paty2019subspace,niles2022estimation,lin2020projection,pmlr-v139-huang21e,lin2021projection}. More generally, one can modify the \citeauthor{Monge1781} problem, by relaxing the push-forward constraint \textit{(i)}, which results in an unbalanced formulation of the \citeauthor{Monge1781} problem~\citep{liero2018optimal,chizat2018unbalanced,yang2018scalable,eyring2024unbalancedness,choi2023generative}, or by considering alternative choices for $c$ in \textit{(ii)}. While OT theory is rife with many cost structures~\citep{ambrosio2003existence,ma2005regularity,lee2012new,figalli2010ma,figalli2010mass}, that choice has comparatively received far less attention in machine learning, where in a majority of applications the cost function is the squared-Euclidean distance over $\R^d$.
\paragraph{Cost structure impacts map structure.}
While some attempts at incorporating Riemannian metrics within OT in ML have been proposed~\citep{cohen2021riemannian, grange2023computational,pooladian2023neural}, computational challenges restrict these approaches to low-dimensional manifolds. We argue in this work that costs that are translation invariant (TI), $c(\bx,\by)\defeq h(\bx-\by)$ with $h:\mathbb{R}^d\rightarrow \mathbb{R}$, offer practitioners a more reasonable middle ground, since a large part of numerical schemes developed for the cost can be extended to TI costs, both for static and dynamic formulations of OT (see e.g., \citet[Chap.7]{villani2009optimal} or \citet{liu2022rectified}).
In particular, 
we propose to focus on \textit{elastic costs}~\citep{zou2005regularization} of the form $h(\bz) = \tfrac12\|\bz\|^2 + \gamma\tau(\bz)$, with $\gamma >0$ and $\tau:\R^d\rightarrow \R$ a regularizer. \citet{cuturi2023monge} show that they result in \citeauthor{Monge1781} maps whose displacements satisfy $T(\bx) - \bx  = - \prox_\tau\circ\nabla f(\bx)$, for some potential $f$. 
In other words, choosing an elastic cost ensures that the displacement of \citeauthor{Monge1781} maps w.r.t. $h$ are shaped by the \emph{proximal operator} of $\tau$.
\begin{figure}[t]
\centering
\includegraphics[width=.245\textwidth]{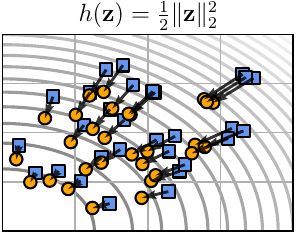}
\includegraphics[width=.245\textwidth]{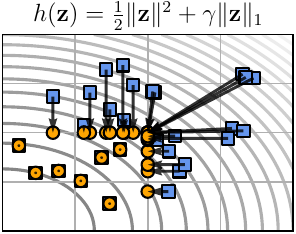}
\includegraphics[width=.245\textwidth]{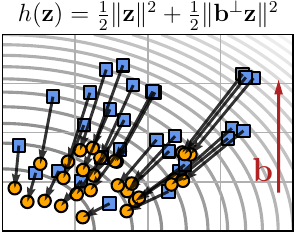}
\includegraphics[width=.245\textwidth]{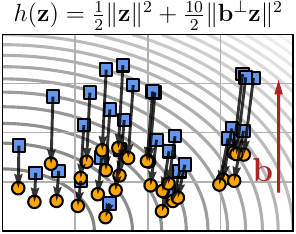}\\
\begin{subfigure}[b]{0.70\textwidth}
    \centering
    \includegraphics[width=\linewidth]{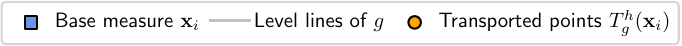}
\end{subfigure}%
    \caption{Illustration of ground truth optimal transport maps with different costs $h$, for the same base function $g$. In this experiment, $g$ is the negative of a random ICNN with 2-dimensional inputs, 3 layers and hidden dimensions of sizes $[8,8,8]$. All plots display the level lines of $g$. The optimal transport map $T_g^h$ are recomputed four times using Prop.~\ref{prop:pushforward}, with four different costs $h$, displayed above each plot. \textit{(left)} When $h$ is the usual $\ell_2^2$ cost, we observe a typical OT map that follows from each $\bx_i$, minus the gradient of $g$. With a sparsity-inducing regularizer \textit{(middle-left)}, we obtain sparse displacements: most arrows follow either of the two canonical axes, some points do not move at all. \textit{(middle-right)} With a cost that penalizes displacements that are orthogonal to a vector $\mathbf{b}$, we obtain displacements that push further to the bottom than in the \textit{(left)} plot. When the penalization strength is increased \textit{(right)}, the displacements are increasingly parallel to $\mathbf{b}$. When $\mathbf{b}$ is not known beforehand, and both source and target measures are given, we present a general procedure that proposes to learn adaptively such a parameter in \S~\ref{sec:learning_structure}.}
    \label{fig:groundtruth_potential_illustration} %
\end{figure}
\paragraph{Contributions.}
While elastic costs offer the promise of obtaining OT maps with prescribed structure, by augmenting the cost with a regularizer $\tau$ with a particular proximal operator, our current understanding of such costs does not extend beyond the early experimentation by \citet{cuturi2023monge}. This should be contrasted with the fine-grained characterization provided by~\citeauthor{Bre91} that a map is optimal for the squared-Euclidean cost if and only if it is the gradient of a convex function. Additionally, using a regularizer raises the question of choosing its hyperparameters. To this end:
\begin{itemize}[leftmargin=.3cm,itemsep=.0cm,topsep=0cm,parsep=2pt]
\item We show in \S~\ref{sec: ground_truth_displacement} that OT maps can be generated for any elastic cost $h$ by running a proximal gradient descent scheme, through the proximal operator of $\tau$, on a suitable objective. This results in, to our knowledge, the first visualization of \citeauthor{Monge1781} maps that extend beyond the usual grad-convex \citeauthor{Bre91} maps for $\ell^2_2$ costs (see Figure~\ref{fig:groundtruth_potential_illustration}), as well as synthetic generation in high-dimensions;
\item We introduce \textit{subspace} elastic costs in \S\,\ref{sec:subspace}, which promote displacements occurring in a low-dimensional subspace spanned by the line vectors of a matrix $A$, $A\in\mathbb{R}^{p\times d}, AA^T=I_p$, setting $\tau(\bz) \defeq \|A^\perp \bz \|_2^2$. We prove sample-complexity estimates for the MBO estimator~\citep{cuturi2023monge} with this cost, and link it with the spiked transport model~\citep{niles2022estimation}.
\item Since the choice of the regularizer $\tau$ in the elastic cost gives rise to a diverse family of OT maps, whose structural properties are dictated by the choice of regularizer, we consider \emph{parametrized families} $\tau_\theta$, and propose in \S~\ref{sec:learning_structure} a loss to select adaptively a suitable $\theta$. 
\item We illustrate all above results, showing MBO estimator performance, \textit{recovery} of $A$ on the basis of i.i.d.~samples, in both synthetic (using our first contribution) and single-cell data tasks, where we demonstrate an improved predictive ability compared to baseline estimators that do not learn $A$.
\end{itemize}

\section{Background: Optimal transport with Elastic Costs}\label{sec:back}
\paragraph{Monge Problem.}
Let $\mathcal{P}_2(\mathbb{R}^d)$ be the set of probability measures with finite second-order moment. We consider in this work cost functions $c$ of the form $c(\bx,\by)\defeq h(\bx-\by)$, where $h:\R^d\rightarrow \R$ is strictly convex and, to simplify a few computations, symmetric, i.e., $h(\bz)=h(-\bz)$. Given two measures $\mu,\nu\in\mathcal{P}_2(\mathbb{R}^d)$, the \citeauthor{Monge1781} problem \citeyearpar{Monge1781} seeks a map $T:\R^d\rightarrow \R^d$ minimizing an average transport cost, as quantified by $h$, of the form:
 \begin{equation}\label{eq:monge}
     T^\star \defeq \argmin_{T_\sharp\mu = \nu} \int_{\R^d} h(\bx-T(\bx))\, \mu(\mathrm{d}\bx)
 \end{equation}
Because the set of admissible maps $T$ is not convex, solving \eqref{eq:monge} requires taking a detour that involves relaxing \eqref{eq:monge} into the so-called Kantorovich dual and semi-dual formulations, involving respectively two functions (or only one in the case of the semi-dual)\citep[\S 1.6]{santambrogio2015optimal}:
\begin{align}\label{eq:dual}
     (f^\star, g^\star) \defeq \argmax_{\substack{f,g: \R^d\rightarrow \R \\f(\bx)+ g(\by)\leq h(\bx-\by)}} \int_{\R^d} f \mathrm{d}\mu +\int_{\R^d} g\, \mathrm{d}\nu\, = \argmax_{\substack{f: \R^d\rightarrow \R,\\ f\text{ is } h\text{-concave}}} \int_{\R^d} f \mathrm{d}\mu +\int_{\R^d} \bar{f}^h \mathrm{d}\nu\,
\end{align}
A function $f$ is said to be $h$-concave if there exists a function $g$ such that $f$ is the $h$-transform of $g$, i.e., $f=\bar{g}^h$, where %
for any function $g:\R^d\rightarrow \R$, we define its $h$-transform as 
\begin{equation}\label{eq:htransf}
    \bar{g}^h(\bx)\defeq \inf_{\by} h(\bx-\by) - g(\by).  
\end{equation}
 We recall a fundamental theorem in optimal transport~\citep[\S1.3]{santambrogio2015optimal}. Assuming the optimal, $h$-concave, potential for~\eqref{eq:dual}, $f^\star$, is differentiable at $\bx_0$ (this turns out to be a mild assumption since $f^\star$ is a.e. differentiable when $h$ is), we have~\citep{wilfrid1996geometry}:
\begin{equation}\label{eq:brenier}
T^\star(\bx)=\bx - (\nabla h)^{-1}(\nabla f^\star(\bx)) = \bx - \nabla h^*\circ \nabla f^\star(\bx)\,,
\end{equation}
where the convex conjugate of $h$ reads: $h^*(\bw)\defeq \sup_{\bz} \langle\bz,\bw\rangle - h(\bz)\,.$
The classic \citeauthor{Bre91} theorem \citeyearpar{Bre91}, which is by now a staple of OT estimation in machine learning~\citep{korotin2019wasserstein,makkuva2020optimal,korotin2021neural,bunne2023} through input-convex neural networks~\citep{amos2017input}, is a particular example, stating for $h=\frac{1}{2}\|\cdot\|^2_2$, that $T(\bx)=\bx - \nabla f^\star(\bx_0)$, since in this case, $\nabla h=(\nabla h)^{-1}=\Id$, see \citep[Theorem 1.22]{santambrogio2015optimal}.

\paragraph{Maps and Elastic Costs.}
\citet{cuturi2023monge} consider TI costs w.r.t. a regularizer $\tau$: for $\gamma>0$ they study \emph{elastic costs} of the form
\begin{equation}\label{eq:h-def}
h(\bz)=\tfrac{1}{2}\|\bz\|_2^2 + \gamma \tau(\bz),
\end{equation}
and show that the resulting \citeauthor{Monge1781} map is shaped by the proximal operator of $\tau$: 
\begin{equation}\label{eq:brenier2}
T^\star(\bx) = \bx - \prox_{\gamma\tau} \circ \nabla f^\star(\bx) \,, \text{ where } \prox_{\gamma\tau}(\bw) \defeq \argmin_{\bz}\frac12\|\bw-\bz\|^2 + \gamma\tau(\bz)\,.
\end{equation}
 
\paragraph{The MBO Estimator.} While the result above is theoretical, in the sense that is assumes knowledge of an optimal $f^\star$, the MBO estimator proposes to evaluate \eqref{eq:brenier2} with an approximation of $f^\star$, using samples from $\mu$ and $\nu$. %
The estimation of optimal potential functions can be carried out using entropy-regularized transport~\citep{cuturi2013sinkhorn}, resulting in optimal entropic potentials~\citep{pooladian2021entropic}. This involves choosing a regularization strength $\varepsilon>0$, and numerically solving the following dual problem using the~\citeauthor{Sinkhorn64} algorithm \citep[\S~4.2]{Peyre2019computational}:
\begin{align}\label{eq:finitedual}
(\bff^\star, \bg^\star) = D^\star(\mathbf{X},\mathbf{a},\mathbf{Y},\mathbf{b}; h, \varepsilon)
\defeq\argmax_{\bff\in\R^n,\bg\in\R^m} \langle\bff, \ba \rangle + \langle\bg, \bb\rangle  - \varepsilon \langle e^{\frac{\bff}{\varepsilon}}, \mathbf{K} e^{\frac{\bg}{\varepsilon}}\rangle\,.
\end{align}

where $\mathbf{K}_{ij} = [\exp(-h(\bx_i-\by_j)/\varepsilon)]_{ij}$. The entropy-regularized optimal transport matrix associated with that cost $h$ and on those samples can be derived directly from these dual potentials, as
$P^\star(\mathbf{X},\mathbf{a},\mathbf{Y},\mathbf{b}; h, \varepsilon) \in \R^{n\times m}$ with the $(i,j)$ entries
\begin{equation}\label{eq:finiteprimal}
    P^\star(\mathbf{X},\mathbf{a},\mathbf{Y},\mathbf{b}; h, \varepsilon)_{i,j} = \exp\left(\frac{\bff^\star_i + \bg^\star_j -h(\bx_i-\by_j)}{\varepsilon}\right)\,.
\end{equation}
We now introduce the soft-minimum operator, and its gradient, defined for any vector $\bu\in\R^q$ as 
\begin{align*}
    \mymin_\varepsilon(\bu) \defeq - \varepsilon \log \sum_{l=1}^q e^{-\bu_l/\varepsilon},\, \text{and }
    \nabla\!\mymin_\varepsilon(\bu) = \left[\frac{e^{-\bu_k/\varepsilon}}{\sum_{l=1}^q e^{-\bu_l/\varepsilon}}\right]_k.
\end{align*}
Using vectors $(\bff^\star, \bg^\star)$, we can define estimators $f_\varepsilon$ and $g_\varepsilon$ for the optimal dual function ($f^\star, g^\star$):
\begin{equation}\label{eq:entpotentials}\hat{f}_\varepsilon : \bx \mapsto \mymin_\varepsilon([h(\bx-\by_j) - \bg^\star_j]_j)\;,\quad
\hat{g}_\varepsilon : \by \mapsto \mymin_\varepsilon([h(\bx_i-\by) + \bff^\star_i]_i)\,.\end{equation}
Plugging these approximations into~\eqref{eq:brenier2} forms the basis for the MBO estimator,
\begin{defn}[MBO Estimator]\label{def:MBO}
Given data, an elastic cost function $h=\ell^2_2+\gamma \tau$ and solutions to Eq.\eqref{eq:finitedual}, the MBO map estimator \citep{pooladian2021entropic,cuturi2023monge} is given by:
\begin{equation}\label{eq:mbo}
T_\varepsilon(\bx)=\bx - \prox_{\gamma\tau}\Bigl(\bx + \sum_{j=1}^m \bp_{j}(\bx) \left(\gamma\nabla\tau(\bx-\by_j) - \by_j\right)\Bigr),.
\end{equation}
where $\bp(\bx)\defeq\nabla\!\mymin_\varepsilon([h(\bx-\by_j) - \bg^\star_j]_j)$ is a probability vector.
\end{defn}

\section{On Ground Truth Monge Maps for Elastic Costs}\label{sec: ground_truth_displacement}
Our strategy to compute examples of ground truth displacements for any elastic cost $h$ rests on the following theorem, which is a direct consequence of~\citep[Theorem 1.17]{santambrogio2015optimal}. 
\begin{prop}
\label{prop:pushforward}
Let $\mu$ be a measure in $\mathcal{P}(\mathbb{R}^d)$. Consider a potential $g:\mathbb{R}^d\rightarrow \mathbb{R}$ and its $h$-transform as defined in \eqref{eq:htransf}. Additionally, set $T_g^h\defeq \Id - \nabla h^\star\circ \nabla \bar{g}^h$. Then $T_g^h$ is the OT \citeauthor{Monge1781} map for cost $h$ between $\mu$ and $(T_g^h)_\sharp \mu$.
\end{prop}

The ability to compute an OT map for $h$ therefore hinges on the ability to solve numerically the $h$-transform \eqref{eq:htransf} of a potential function $g$. This can be done, provably, as long as $g$ is concave and smooth, and $\mathrm{prox}_{\tau}$ is available, as shown in the following result

\begin{prop}\label{prop: prox_descent}
    Assume $g$ is concave, $L$-smooth, and that $\lambda < 2/ L$.
    Setting $\by=\bx$ and iterating
    \begin{equation}\label{eq:pgd}
\by \leftarrow \bx + \prox_{\frac{\lambda \gamma}{\lambda+1} \tau}\left(\frac{\by - \bx + \lambda \nabla g(\by)}{1 + \lambda}\right)\end{equation}
 converges to a point $\by^\star(\bx)= \argmin_\by h(\bx - \by) - g(\by)$.
Furthermore, we have
\begin{align}
\label{eq:grad_from_pgd}
    \bar{g}^h(\bx) = h(\bx-\by^\star(\bx)) - g(\by^\star(\bx))\,,\quad\text{and}\quad
    \nabla \bar{g}^h(\bx) = \nabla h(\bx-\by^\star(\bx)).
\end{align}
\end{prop}
\begin{proof}
    Because $h$ is the sum of a function $\gamma \tau$ and a quadratic term, one has, thanks to \citep[\S2.1.1]{parikh2014proximal}, that the proximal operator of $\lambda h$ can be restated in terms of the proximal operator of $\tau$ as shown in the equation. The convergence of iterates~\eqref{eq:pgd} follows from \citep[Thm. 1]{beck2009fast} or \citep[Thm. 1]{rockafellar1976monotone}. The final identities are given by \citep[Prop. 18.7]{bauschke2011convex}.
\end{proof}

In summary, the proximal operator of $\tau$ is the only thing needed to implement iterations~\eqref{eq:pgd}, and, as a result, the $h$-transform of a suitable concave potential. We can then plug the solution in~\eqref{eq:grad_from_pgd} to compute the quantities of interest numerically, plugged back again in a proximal operator to compute the pushforward $T_g^h$.
In practice, we use the JAXOPT~\citep{jaxopt_implicit_diff} library to integrate these steps in our differentiable pipeline seamlessly.
We illustrate numerically in $2$d the resulting transport maps for different choices of regularizer $\tau$ in Fig.~\ref{fig:groundtruth_potential_illustration}. 
In this illustration, we use the same base function $g$, so we see clearly that the choice of cost has a drastic impact on the form of the transport map.

\begin{figure*}
\centering
\includegraphics[width=.245\textwidth]{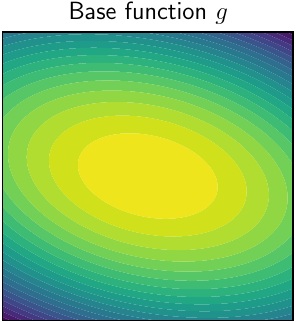}
\includegraphics[width=.245\textwidth]{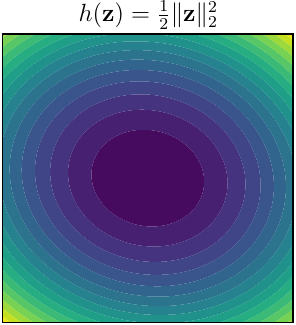}
\includegraphics[width=.245\textwidth]{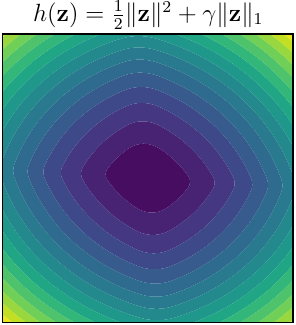}
\includegraphics[width=.245\textwidth]{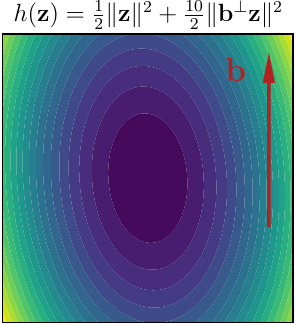}
    \caption{Illustration of the $h$-transform computation in $2$d. \textit{(left)}: base concave potential $g$, here a negative quadratic. \textit{(other figures)} Level lines of the corresponding $h$-transform $\bar{g}^h$ for different choices of $h$. The $h$-transform is computed using the iterations described in Prop.~\autoref{prop: prox_descent}.}
    \label{fig:h_transforms_illustration} 
\end{figure*}

\section{Subspace Elastic Costs}\label{sec:subspace}
Recall that for a rank-$p$ matrix $A\in\R^{p\times d}$, $p\leq d$, the projection matrix that maps it to its orthogonal is $A^\perp = I-A^T(AA^T)^{-1}A$. When $A$ lies in the Stiefel manifold (i.e. $AA^T =I$), we have the simplification $A^\perp = I-A^T A$. This results in the Pythagorean identity $\|\bz\|^2 = \|A^\perp \bz\|^2+\|A\bz\|^2$. In order to promote displacements that happen \textit{within} the span of $A$, we must set a regularizer that penalizes the presence of $\bz$ within its \emph{orthogonal complement}, namely 
\begin{equation}\label{eq:tauorth}\tau_{A^\perp}(\bz)\defeq \tfrac{1}{2}\|A^\perp \bz\|^2_2.\end{equation}
Since $\tau_{A^\perp}$ is quadratic, its proximal operator can be obtained by solving a linear system~\cite[\S6.1.1]{parikh2014proximal}; developing and using the matrix inversion lemma results in two equivalent quantities 
\begin{equation}\label{eq:prox_orth}
\prox_{\gamma\tau_{{A^\perp}}}(\bz) = \left(I_d + \gamma (A^\perp)^T A^\perp\right)^{-1}\bz = 
\tfrac{1}{1+\gamma}(I_d+\gamma A^T(AA^T)^{-1}A)\,\bz .
\end{equation}

To summarize, given an orthogonal sub-basis $A$ of $p$ vectors (each of size $d$), promoting that a vector $\bz$ lies in its orthogonal can be achieved by regularizing its norm in the space orthogonal to the span of $A$. That norm has a proximal operator that can be computed by parameterizing $A$ \textit{explicitly}, either as a full-rank $p\times d$ matrix, or more simply a $p\times d$ orthogonal matrix, to recover the suitable proximal operator for $\tau_{A^\perp}$ in \eqref{eq:prox_orth}. Because that operator is simpler when $A\in\mathcal{S}_{p,d}$ is in the Stiefel manifold,
\begin{equation}\label{eq:prox_orth_stiefel}
\prox_{\gamma\tau_{{A^\perp}}}(\bz) = \frac{1}{1+\gamma}\left(I_d + \gamma A^T A\right)\bz.\\
\end{equation}
We propose to restrict the study in this work to elastic costs of the form~\ref{eq:prox_orth} where $A\in\mathcal{S}_{p,d}$. We also present in \Cref{sec:appendix_subspace_elastic_costs} alternative parameterizations left aside for future work.

\subsection{Statistical Aspects of Subspace Monge Maps}\label{sec: statistical_complexity}
The family of costs \eqref{eq:tauorth} is designed to promote transport maps whose displacements mostly lie in a low-dimensional subspace of $\R^d$.
In this section, we consider the statistical complexity of estimating such maps from data, assuming $A$ is known.
The question of estimating transport maps was first studied in a statistical context by \citet{hutter2021minimax}, and subsequent research has proposed alternative estimation procedures, with different statistical and computational properties~\citep{deb2021rates, manole2021plugin,muzellec2021near,pooladian2021entropic}.
We extend this line of work by considering the analogous problem for Monge maps with structured displacements.

We show that with a proper choice of $\varepsilon$, the MBO estimator outlined in Definition~\ref{def:MBO} is a consistent estimator of $T^\star$ as $n \to \infty$, and prove a rate of convergence in $L^2(\mu)$.
We also give preliminary theoretical evidence that, as $\gamma \to \infty$, maps corresponding to the subspace structured cost $\frac 12 \ell_2^2 + \gamma \tau_{A^\perp}$ can be estimated at a rate that depends only on the subspace dimension $p$, rather than on the ambient dimension $d$, thereby avoiding the \emph{curse of dimensionality}.

\paragraph{Sample Complexity Estimates for the MBO Estimator.}
The MBO estimator is a generalization of the entropic map estimator, originally defined by \citet{pooladian2021entropic} for the quadratic cost $h = \frac 12 \ell_2^2$. This estimator has been statistically analyzed in several regimes, see e.g., \citep{pooladian2023minimax, rigollet2022sample,del2022improved} and \citep{goldfeld2022limit}. We show that this procedure also succeeds for subspace structured costs of the form $h = \frac 12 \ell_2^2 + \gamma \tau_{A^\perp}$. As a result of being recast as an estimation task for quadratic cost, the following sample-complexity result for the MBO estimator follows from \cite[Theorem 3]{pooladian2021entropic}, and a computation relating the MBO estimator to a barycentric projection for the costs we consider (see \Cref{app: statistical_proofs} for the full statements and proofs).
\begin{thm}\label{thm: stat_complexity_main}
Let $A \in \R^{p \times d}$ be fixed, and consider $\tilde{T}$ given by an $M$-smooth and $m$-strongly convex function, whose inverse has at least three derivatves, and suppose $\nu$ has an upper- and lower- bounded density, and $\mu$ is upper-bounded, both supported over $\Omega \subseteq \R^d$ compact. Consider $T^\star$ of the form \cref{eq: pre_cond_map} for some $\gamma \geq 0$ fixed, and suppose we have samples $X_1,\ldots,X_n \sim \mu$ and $Y_1,\ldots, Y_n \sim (T^\star)_\sharp \mu$. Let $\hat{T}_\eps$ be the MBO estimator with $\eps \asymp n^{-\frac{1}{d + 4}}$. Then it holds that
\begin{align*}
    \mathbb{E}\|\hat{T}_\eps - T^\star\|^2_{L^2(\mu)} \lesssim n^{-\frac{2}{d + 4}}\,,
\end{align*}
where the underlying constants depend on properties of $\mu,\nu,\tilde{T}, \gamma$ and $A$.
\end{thm}

\subsection{Connection to the Spiked Transport Model}
The additional structure we impose on the displacements allows us to closely relate our model to the ``spiked transport model" as defined by \citet{niles2022estimation}.
The authors studied the estimation of the Wasserstein distance in the setting where the Brenier map between $\mu$ and $\nu$ takes the form, \looseness=-1
\begin{align}\label{eq: T_spiked}
  T_{\text{spiked}}(\bx) = \bx - A^T (A \bx - S(A \bx))\,,  
\end{align}
where $A \in \mathcal{S}_{p, d}$ and $S: \R^p \to \R^p$ is the gradient of a convex function on $\R^p$.
\citet{divol2022optimal} performed a statistical analysis of the map estimation problem under the spiked transport model. They constructed an estimator $\hat{T}_n$ such that the $L^2(\mu)$ risk decays with respect to the \emph{intrinsic dimension} $p \ll d$; this is summarized in the following theorem.
\begin{thm}[{\citealp[Section 4.6]{divol2022optimal}}]\label{thm: spiked_map_rates}
    Suppose $\Tilde{T}$ is bi-Lipschitz (i.e., is the gradient of a smooth, and strongly convex function) and $\mu$ has compact support, with density bounded above and below. Suppose further that there exists a matrix $A \in \R^{p \times d}$ on the Stiefel manifold such that $\nu \defeq (T_{\text{spiked}})_\sharp \mu$, with $T_{\text{spiked}}$ defined as in \cref{eq: T_spiked}.
    Assume that $\mu$ is known explicitly.
    Given $n$ i.i.d.\ samples from $\nu$, there exists an estimator $\hat T_n$ satisfying
    \begin{align}
        \mathbb{E}\| \hat{T}_n - T_{\text{spiked}}\|^2_{L^2(\mu)} \lesssim_{\log(n)} n^{-\Theta(\frac 1p )}\,.
    \end{align}
\end{thm}

We now argue that the spiked transport model can be recovered in the large $\gamma$ limit of subspace structured costs.
Indeed, if $\gamma \to \infty$, then displacements in the subspace orthogonal to $A$ are heavily disfavored, so that the optimal coupling will concentrate on the subspace given by $A$, thereby recovering a map of the form \eqref{eq: T_spiked}, which by Theorem~\ref{thm: spiked_map_rates} can be estimated at a rate independent of the ambient dimension.
Making this observation quantitative by characterizing the rate of estimation of $T^\star$ as a function of $\gamma$ for $\gamma$ large is an interesting question for future work.

\section{A Bilevel Loss to Learn Elastic Costs}\label{sec:learning_structure}
Following \S~\ref{sec:subspace}, we propose a general loss to \textit{learn} the parameter $\theta$ of a family of regularizers $\{\tau_\theta\}_\theta$ given source and target samples only. Our goal is to infer adaptively a $\theta$ that promotes regular displacements, apply it within the estimation of \citeauthor{Monge1781} maps using MBO, and leverage this knowledge to improve prediction quality.
Given input and target measures characterized by point clouds $\mathbf{X},\mathbf{Y}$ and probability weights $\ba,\bb$, our loss follows a simple intuition: the ideal parameter $\theta$ should be such that the bulk of the OT cost bore by the optimal \citeauthor{Monge1781} map, for that cost, is dominated by displacements that have a \textit{low} regularization value. Since the only moving piece in our pipeline will be $\theta$, we consider all other parameters \textit{constant} in the computation of the primal solution, to re-write \eqref{eq:finiteprimal} as:
\begin{equation}
P^\star(\theta) \defeq P^\star\left(\mathbf{X},\mathbf{a},\mathbf{Y},\mathbf{b}; \tfrac{1}{2}\ell^2_2 + \gamma \tau_{\theta}, \varepsilon\right) \in\mathbb{R}^{n\times m}.   
\end{equation}
Each entry $[P^\star(\theta)]_{ij}$ quantifies the optimal association strength between a pair $(\bx_i,\by_j)$ when the cost is parameterized by $\theta$, where a given pair can be encoded as a displacement $\bz_{ij}\defeq\by_j - \bx_i$. For the regularizer $\theta$ to shape displacements, we expect $P^\star(\theta)$ to have a large entry on displacements $\bz_{ij}$ that exhibit a low regularizer $\tau_\theta(\bz_{ij})$ value. In other words, we expect that $\tau_\theta(\bz_{ij})$ to be as small as possible when $P_{ij}^\star(\theta)$ is high. We can therefore consider the loss
\begin{defn}[Elastic Costs Loss]\label{defn:loss} Given two weighted point clouds $\ba,\bX,\bb,\bY$, and $P^\star(\theta)$ defined implicitly, as an OT solution in Equation~\eqref{eq:finiteprimal},  let
\begin{equation}\label{eq:cost_function}
\mathcal{L}(\theta)\defeq \left\langle P^\star(\theta), R(\theta)\right\rangle, \text{ with } [R(\theta)]_{ij} = \tau_{\theta}(\bz_{ij}).
\end{equation}
\end{defn}
Because $P^\star(\theta)$ is itself obtained as the solution to an optimization problem, minimizing $\mathcal{L}$ is, therefore, a \emph{bilevel} problem.
To solve it, we must compute the gradient $\nabla \mathcal{L}(\theta)$, given by the vector-Jacobian operators $\partial P^\star(\cdot)^*[\cdot]$ and $\partial R(\cdot)^*[\cdot]$ of $P^\star$ and $R$ respectively,
borrowing notations from~\citep[\S2.3]{blondel2024elements} (see also \S~\ref{sec:explainloss} for a walk-through of this identity)
\begin{equation}
\label{eq:chain_rule2}
    \nabla \mathcal{L}(\theta) = \partial P^\star(\theta)^*[R(\theta)] + \partial\!R(\theta)^*[P^\star(\theta)]
\end{equation}
The first operator $\partial P^\star(\cdot)^*[\cdot]$ requires differentiating the solution to an optimization problem $P^\star(\theta)$. This can be done~\citep[\S10.3.3]{blondel2024elements} using either unrolling of \citeauthor{Sinkhorn64} iterations or using implicit differentiation. We rely on \textsc{OTT-JAX}~\citep{cuturi2022optimal} to provide that operator, using unrolling. The second operator $\partial R(\cdot)^*[\cdot]$  can be trivially evaluated, since it only involves differentiating the regularizer function $\tau_\theta(\cdot)$, which can be done using automatic differentiation.

\textbf{Learning Subspace Costs. } We focus in this section on the challenges arising when optimizing subspace costs, as detailed in Section~\ref{sec:learning_structure}.
Learning matrix $A$ in this context is equivalent to learning a subspace in which the displacement between the source and target measures happen mostly in the range of $A$.
As discussed previously, the cost function $\mathcal{L}(A)$ should be optimized over the Stiefel manifold~\citep{edelman1998geometry}.
We use Riemannian gradient descent~\citep{boumal2023introduction} for this task, which iterates, for a step-size $\eta > 0$ 
$$A \xleftarrow{}\mathcal{P}(A - \eta \tilde\nabla\mathcal{L}(A))\,, $$ with the \emph{Riemannian gradient} of $\mathcal{L}$ given by $\tilde\nabla\mathcal{L}(A) \defeq G -  A G^T A$ where: $G \defeq \nabla \mathcal{L}(A)$ the standard Euclidean gradient of $A$ computed with automatic differentiation provided in~\eqref{eq:chain_rule2}; $\mathcal{P}$ is the projection on the Stiefel manifold, with formula $\mathcal{P}(A) = (AA^{\top})^{-1/2}A$. These updates ensure that one stays on the manifold~\citep{absil2012projection}.

\section{Experiments}
Thanks to our ability to compute ground-truth $h$-optimal maps presented in \S~\ref{sec: ground_truth_displacement}, we generate benchmark tasks to measure the performance of \citeauthor{Monge1781} map estimators. We propose in \S~\ref{subsec:groundtruth} to test the MBO estimator~\citep{cuturi2023monge} when the ground-truth cost $h$ that has generated those benchmarks is known. In \S\ref{subsec:learningsub}, we consider the more difficult task of learning simultaneously, and as outlined in \S~\ref{sec:learning_structure}, an OT map and the ground truth parameter of a subspace-elastic cost defined by a matrix $A^*$ of size $p^*\times d$. The cost is parameterized by a matrix $\hat A$ of size $\hat p\times d$, where $\hat p$ is equal to or larger than $p^*$. We check with this synthetic task the soundness of the loss $\mathcal{L}(\theta)$, Definition~\ref{defn:loss}, and of our Riemannian descent approach by evaluating to what extent the $\hat p$ vectors in $\hat{A}$ recovers the subspace spanned by $A^*$. Finally, we consider in \S~\ref{subsec:ss} a direct application of subspace elastic costs to real data, without any ground truth knowledge, using perturbations of single-cell data. In this experiment, our pipeline learns both an OT map and a subspace. Our code implements a parameterized \texttt{RegTICost} class, added to \textsc{OTT-JAX} \citep{cuturi2022optimal}. Such costs can be fed into the \texttt{Sinkhorn} solver, and their output cast as \texttt{DualPotentials} objects that can output the $T_\varepsilon$ map given in Definition~\ref{def:MBO}.

\subsection{MBO on Synthetic Ground Truth Displacement}\label{subsec:groundtruth}
\begin{figure}[t]
   \centering
\includegraphics[width=\textwidth]{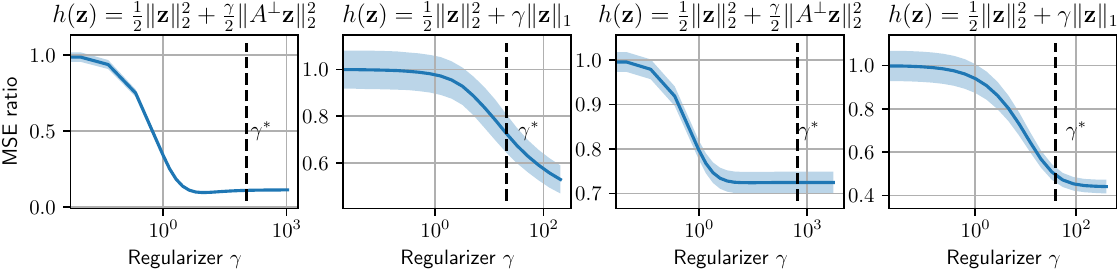}
\caption{Performance of the MBO estimator on two ground-truth tasks involving the $\tau=\ell_1$ and $\tau_{A^\perp}=\|A^\perp\bz\|^2_2$ structured costs, where $p=2$ in dimension $d=5$ (two figures to the \textit{left}) and dimension $d=10$ (two figures to the \textit{right}). We display the MSE ratio between the MSE estimated with a regularizer strength $\gamma > 0$ and that in the absence of regularization (i.e., $\gamma=0$). The level of regularization used for generating the ground truth data is $\gamma^*$, whereas performance are shown varying w.r.t. $\gamma$. 
We display curves $\pm$ s.t.d. estimated over 10 random seeds.}\label{fig:mse}
\end{figure}
In this section, we assume that the regularizer $\tau$ is \textit{known}, using the same $\tau$ both for generation and estimation, but that the ground truth regularization strength $\gamma^*$ used to generate data is not known. We use that cost to evaluate the transport associated with $\bar g_\varepsilon^h$ on a sample of points, using Proposition\ref{prop:pushforward}, and then compare the performance of Sinkhorn based estimators, either with that cost or the standard $\tfrac12\ell^2_2$ cost (which corresponds to $\gamma=0$). 

We consider the $\tau=\ell_1$ and $\tau_{A^\perp}=\|A^\perp\bz\|^2_2$ regularizer, and its associated proximal soft-thresholding operator. While we assume knowledge of $\tau$ in the MBO estimator, we do not use the ground truth regularization $\gamma^*$ which generated the data, and, instead, vary it. The data is generated following \S~\ref{sec: ground_truth_displacement} by sampling a concave quadratic function $g(\bz)\defeq-\frac{1}{2}(\bz-\bw)^T M (\bz-\bw)$ where $M$ is a Wishart matrix, sampled as $M=QQ^T$, where $Q\in\R^{d\times 2d}$ is multivariate Gaussian, and $\bw$ is a random Gaussian vector.
We then sample $n=1024$ Gaussian points stored in $\bX_{T}$ and transport each using the map defined in Proposition~\ref{prop:pushforward}, computed in practice with Proposition~\ref{prop: prox_descent}. This recovers matched train data $\bX_{T}$ and $\bY_{T}$. We do the same for a test fold $\bX_{t},\bY_{t}$ of the same size, to report our metric, the mean squared error (MSE), defined as $\|T_{\varepsilon}(\bX_t) - \bY_t\|^2_2$, where $T_\varepsilon$ is obtained from Definition~\ref{def:MBO} using $\bX_{T},\bY_{T}$. We plot this MSE as a function of $\gamma$, where $\gamma=0$ corresponds exactly to the MBO using the naked $\ell_2^2$ cost.
We observe in Figure~\ref{fig:mse} that the MBO estimator with positive $\gamma$ outperforms significantly that using $\ell_2^2$ only, for any range of the parameter $\gamma$.

\begin{figure}
\centering
    \centering
    \includegraphics[width=.98\textwidth]{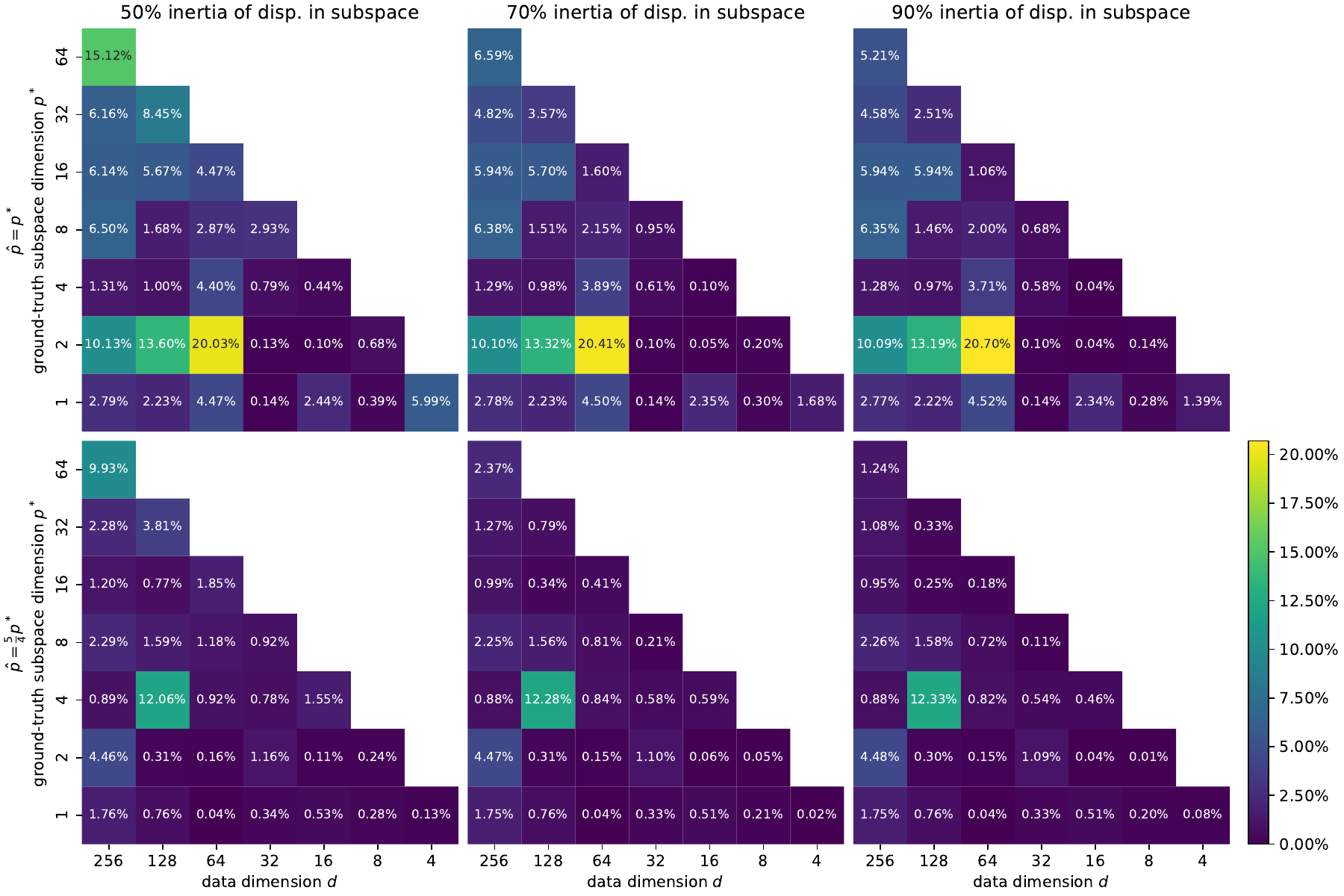}
    \label{exp:recovery}
\caption{Error averaged over 5 seeded runs (lower is better) in $[0,1]$ of the $\hat p\times d$ orthogonal matrix $\hat A$ recovered by our algorithm, compared to the ground truth $p^*\times d$ cost matrix $A^*$. Error bars are not shown for compactness, but are negligible since all quantities are bounded below and close to $0$. Dimensions $d, p^*$ vary in each of these 6 plots, whereas $\hat p$ is fixed to either $p^*$ (top row) or $1.25 p^*$ (bottom row).
Error is quantified as the normalized squared-residual error obtained when projecting the $p^*$ basis vectors of $A^*$ onto the span of $\hat{A}$. From left to right, the regularization strength $\gamma^*$ increases to ensure that 50\%, 70\% and 90\% of the total inertia of all displacements generated by the ground-truth \citeauthor{Monge1781} map are borne by the $p^*$ highest singular values. As expected, recovery is easier when $\hat p$ is slightly larger than $p^*$ (bottom) compared to being exactly equal (top). It is also easier as the share of inertia captured by $p^*$ increases.}
\label{exp:figure-recovery}
\end{figure}

\subsection{Recovery of Ground-Truth Subspace Parameters in Elastic Costs }\label{subsec:learningsub}
We propose to test the ability of our pipeline to recover the ground truth $A^*$ parameter of a regularizer $\tau_{A^\perp}$  as defined in \eqref{eq:tauorth}. To do so, we proceed as follows: For dimension $d$, we build the ground truth cost $h$ by selecting $A^*$, sampling a $p^*\times d$ normal matrix that is then re-projected on the $p^*\times d$ Stiefel manifold.
Next, we sample a random ICNN, and set the base function $g$ to be its negative. We then sample a point cloud $\bX$ of $n=512$ standard Gaussian points, and apply, following Proposition \ref{prop:pushforward}, the corresponding ground-truth transport to obtain $\bY$ of the same size. We tune the regularization parameter $\gamma$ for $\tau$, to ensure that the $p^*$ first singular values of displacements $\bY-\bX$ captured either 50\%, 70\% or 90\% of the total inertia. We expect that the larger this percentage, the easier recovery should be. See \S~\ref{app:exps} for details.

We then launch our solver with a dimension that is either $\hat{p}=p^*$ or $\hat{p}=\tfrac54 p^*$. We measure recovery of $A^*$ by $\hat{A}$ through the average (normalized by the basis size) of the residual error, when projecting the vectors in $A^*$ in the span of the basis $\hat{A}$, namely $\|A^*-\hat{A}\hat{A}^TA^*\|_2^2/p^*$. For simplicity, we report performance after 1000 iterations of Riemannian gradient descent, with a stepsize $\eta$ of $0.1/\sqrt{i + 1}$ at iteration $i$. All results in Figure~\ref{exp:figure-recovery} agree with intuition in the way performance varies with $d, p^*, \hat{p}$. More importantly, with errors that are often below one percent, we can be confident that our algorithm is sound. We observe that most underperforming experiments could be improved using early stopping.

\subsection{Learning Displacement Subspaces for Single-Cell Transport}\label{subsec:ss}
We borrow the experimental setting in \citep{cuturi2023monge}, using single-cell RNA sequencing data borrowed from \citep{srivatsan2020massively}. The original dataset shows the responses of cancer cell lines to 188 drug perturbations, downsampled to the $5$ drugs (Belinostat, Dacinostat, Givinostat, Hesperadin, and Quisinostat) that have the largest effect. After various standard pre-processings (dropping low-variability genes, and using a $\log(1+\cdot)$ scaling), we project the dataset to $d=256$ directions using PCA. We then use 80\% train/ 20\%test folds to benchmark two MBO estimators: that computed using the $\ell_2^2$ cost, and ours, using an elastic subspace cost, following the learning pipeline outlined in \S~\ref{sec:learning_structure}. We plot the Sinkhorn divergence (cf. \citet{feydy2019interpolating}) for the $\ell_2^2$ cost (see the documentation in \textsc{OTT-JAX} \citep{cuturi2022optimal}).

\begin{figure}[H]
   \centering
\includegraphics[width=\textwidth]{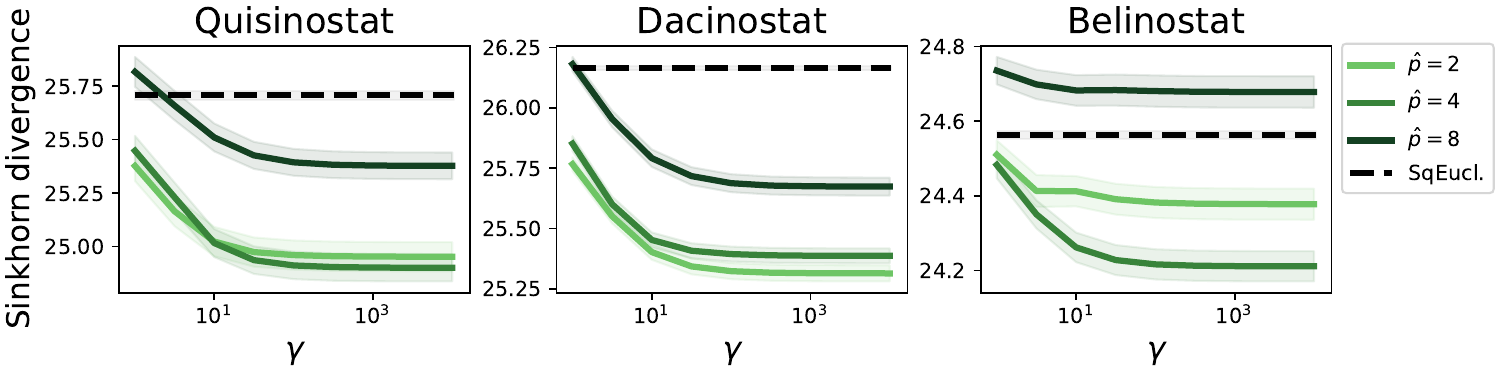}
\caption{Predictive performance of the MBO estimator on single-cell datasets, $d=256$, using either the naive baseline $\ell_2^2$ cost (black dotted line) or elastic subspace cost \eqref{eq:tauorth}, with varying $\gamma$ and $\hat{p}$. 
Remarkably, promoting displacements to happen in a subspace of much lower dimension improves predictions, even when measured in the squared-Euclidean distance.
}\label{fig:mse}
\end{figure}
\vspace{-5mm}
\textbf{Conclusion.}
In this work, we proposed an algorithmic mechanism to design ground-truth transports for structured costs. As a first application, we were able to successfully benchmark the MBO estimator of \citep{cuturi2023monge} on two tasks (involving the $\ell_1$ and an orthogonal projection norm), showcasing the versatility of the MBO framework. Next, we demonstrated our ability to leverage subspace-penalizing costs to \emph{learn displacement subspaces} by solving an \emph{inverse OT problem}. 
We showed successful numerical performance of the MBO estimator when the subspace is known but the regularization strength is not, but also that we were able to learn the ground-truth subspace. We foresee several open directions, the most encouraging being cementing connections between subspace regularized transport, and the spiked transport model.

\bibliography{biblio}

\begin{thebibliography}{59}
\providecommand{\natexlab}[1]{#1}
\providecommand{\url}[1]{\texttt{#1}}
\expandafter\ifx\csname urlstyle\endcsname\relax
  \providecommand{\doi}[1]{doi: #1}\else
  \providecommand{\doi}{doi: \begingroup \urlstyle{rm}\Url}\fi

\bibitem[Absil and Malick(2012)]{absil2012projection}
P-A Absil and J{\'e}r{\^o}me Malick.
\newblock Projection-like retractions on matrix manifolds.
\newblock \emph{SIAM Journal on Optimization}, 22\penalty0 (1):\penalty0
  135--158, 2012.

\bibitem[Ambrosio and Pratelli(2003)]{ambrosio2003existence}
Luigi Ambrosio and Aldo Pratelli.
\newblock Existence and stability results in the l1 theory of optimal
  transportation.
\newblock \emph{Optimal Transportation and Applications: Lectures given at the
  CIME Summer School, held in Martina Franca, Italy, September 2-8, 2001},
  2003.

\bibitem[Amos et~al.(2017)Amos, Xu, and Kolter]{amos2017input}
Brandon Amos, Lei Xu, and J~Zico Kolter.
\newblock {Input Convex Neural Networks}.
\newblock volume~34, 2017.

\bibitem[Bauschke and Combettes(2011)]{bauschke2011convex}
Heinz~H Bauschke and Patrick~L Combettes.
\newblock \emph{Convex analysis and monotone operator theory in Hilbert
  spaces}, volume 408.
\newblock Springer, 2011.

\bibitem[Beck and Teboulle(2009)]{beck2009fast}
Amir Beck and Marc Teboulle.
\newblock A fast iterative shrinkage-thresholding algorithm for linear inverse
  problems.
\newblock \emph{SIAM journal on imaging sciences}, 2\penalty0 (1):\penalty0
  183--202, 2009.

\bibitem[Blondel and Roulet(2024)]{blondel2024elements}
Mathieu Blondel and Vincent Roulet.
\newblock The elements of differentiable programming.
\newblock \emph{arXiv preprint arXiv:2403.14606}, 2024.

\bibitem[Blondel et~al.(2021)Blondel, Berthet, Cuturi, Frostig, Hoyer,
  Llinares-L{\'o}pez, Pedregosa, and Vert]{jaxopt_implicit_diff}
Mathieu Blondel, Quentin Berthet, Marco Cuturi, Roy Frostig, Stephan Hoyer,
  Felipe Llinares-L{\'o}pez, Fabian Pedregosa, and Jean-Philippe Vert.
\newblock Efficient and modular implicit differentiation.
\newblock \emph{arXiv preprint arXiv:2105.15183}, 2021.

\bibitem[Boumal(2023)]{boumal2023introduction}
Nicolas Boumal.
\newblock \emph{An introduction to optimization on smooth manifolds}.
\newblock Cambridge University Press, 2023.

\bibitem[Brenier(1991)]{Bre91}
Yann Brenier.
\newblock Polar factorization and monotone rearrangement of vector-valued
  functions.
\newblock \emph{Communications on Pure and Applied Mathematics}, 44\penalty0
  (4), 1991.
\newblock \doi{10.1002/cpa.3160440402}.

\bibitem[Bunne et~al.(2023)Bunne, Stark, Gut, del Castillo, Levesque, Lehmann,
  Pelkmans, Krause, and R{\"a}tsch]{bunne2023}
Charlotte Bunne, Stefan~G. Stark, Gabriele Gut, Jacobo~Sarabia del Castillo,
  Mitch Levesque, Kjong-Van Lehmann, Lucas Pelkmans, Andreas Krause, and Gunnar
  R{\"a}tsch.
\newblock Learning single-cell perturbation responses using neural optimal
  transport.
\newblock \emph{Nature Methods}, 20, 2023.

\bibitem[Chizat et~al.(2018)Chizat, Peyr{\'e}, Schmitzer, and
  Vialard]{chizat2018unbalanced}
Lenaic Chizat, Gabriel Peyr{\'e}, Bernhard Schmitzer, and Fran{\c{c}}ois-Xavier
  Vialard.
\newblock Unbalanced optimal transport: Dynamic and kantorovich formulations.
\newblock \emph{Journal of Functional Analysis}, 274\penalty0 (11):\penalty0
  3090--3123, 2018.

\bibitem[Choi et~al.(2023)Choi, Choi, and Kang]{choi2023generative}
Jaemoo Choi, Jaewoong Choi, and Myungjoo Kang.
\newblock Generative modeling through the semi-dual formulation of unbalanced
  optimal transport.
\newblock \emph{Advances in Neural Information Processing Systems}, 36, 2023.

\bibitem[Cohen et~al.(2021)Cohen, Amos, and Lipman]{cohen2021riemannian}
Samuel Cohen, Brandon Amos, and Yaron Lipman.
\newblock Riemannian convex potential maps.
\newblock In \emph{International Conference on Machine Learning}, pages
  2028--2038. PMLR, 2021.

\bibitem[Cuturi(2013)]{cuturi2013sinkhorn}
Marco Cuturi.
\newblock Sinkhorn distances: Lightspeed computation of optimal transport.
\newblock In \emph{Advances in neural information processing systems}, pages
  2292--2300, 2013.

\bibitem[Cuturi et~al.(2022)Cuturi, Meng-Papaxanthos, Tian, Bunne, Davis, and
  Teboul]{cuturi2022optimal}
Marco Cuturi, Laetitia Meng-Papaxanthos, Yingtao Tian, Charlotte Bunne, Geoff
  Davis, and Olivier Teboul.
\newblock Optimal transport tools (ott): A jax toolbox for all things
  wasserstein.
\newblock \emph{arXiv preprint arXiv:2201.12324}, 2022.

\bibitem[Cuturi et~al.(2023)Cuturi, Klein, and Ablin]{cuturi2023monge}
Marco Cuturi, Michal Klein, and Pierre Ablin.
\newblock Monge, bregman and occam: Interpretable optimal transport in
  high-dimensions with feature-sparse maps.
\newblock In \emph{Proceedings of the 40th ICML}, 2023.

\bibitem[Deb et~al.(2021)Deb, Ghosal, and Sen]{deb2021rates}
Nabarun Deb, Promit Ghosal, and Bodhisattva Sen.
\newblock Rates of estimation of optimal transport maps using plug-in
  estimators via barycentric projections.
\newblock \emph{arXiv preprint arXiv:2107.01718}, 2021.

\bibitem[del Barrio et~al.(2022)del Barrio, Gonzalez-Sanz, Loubes, and
  Niles-Weed]{del2022improved}
Eustasio del Barrio, Alberto Gonzalez-Sanz, Jean-Michel Loubes, and Jonathan
  Niles-Weed.
\newblock An improved central limit theorem and fast convergence rates for
  entropic transportation costs.
\newblock \emph{arXiv preprint arXiv:2204.09105}, 2022.

\bibitem[Divol et~al.(2022)Divol, Niles-Weed, and Pooladian]{divol2022optimal}
Vincent Divol, Jonathan Niles-Weed, and Aram-Alexandre Pooladian.
\newblock Optimal transport map estimation in general function spaces.
\newblock \emph{arXiv preprint arXiv:2212.03722}, 2022.

\bibitem[Edelman et~al.(1998)Edelman, Arias, and Smith]{edelman1998geometry}
Alan Edelman, Tom{\'a}s~A Arias, and Steven~T Smith.
\newblock The geometry of algorithms with orthogonality constraints.
\newblock \emph{SIAM journal on Matrix Analysis and Applications}, 20\penalty0
  (2):\penalty0 303--353, 1998.

\bibitem[Eyring et~al.(2024)Eyring, Klein, Uscidda, Palla, Kilbertus, Akata,
  and Theis]{eyring2024unbalancedness}
Luca Eyring, Dominik Klein, Th{\'e}o Uscidda, Giovanni Palla, Niki Kilbertus,
  Zeynep Akata, and Fabian~J Theis.
\newblock Unbalancedness in neural monge maps improves unpaired domain
  translation.
\newblock In \emph{The Twelfth International Conference on Learning
  Representations}, 2024.
\newblock URL \url{https://openreview.net/forum?id=2UnCj3jeao}.

\bibitem[Feydy et~al.(2019)Feydy, S{\'e}journ{\'e}, Vialard, Amari, Trouv{\'e},
  and Peyr{\'e}]{feydy2019interpolating}
Jean Feydy, Thibault S{\'e}journ{\'e}, Fran{\c{c}}ois-Xavier Vialard, Shun-ichi
  Amari, Alain Trouv{\'e}, and Gabriel Peyr{\'e}.
\newblock Interpolating between optimal transport and mmd using sinkhorn
  divergences.
\newblock In \emph{The 22nd International Conference on Artificial Intelligence
  and Statistics}, pages 2681--2690. PMLR, 2019.

\bibitem[Figalli and Rifford(2010)]{figalli2010mass}
Alessio Figalli and Ludovic Rifford.
\newblock Mass transportation on sub-riemannian manifolds.
\newblock \emph{Geometric and functional analysis}, 20:\penalty0 124--159,
  2010.

\bibitem[Figalli et~al.(2010)Figalli, Rifford, and Villani]{figalli2010ma}
Alessio Figalli, Ludovic Rifford, and C{\'e}dric Villani.
\newblock On the ma--trudinger--wang curvature on surfaces.
\newblock \emph{Calculus of Variations and Partial Differential Equations},
  39:\penalty0 307--332, 2010.

\bibitem[Gangbo and McCann(1996)]{wilfrid1996geometry}
Wilfrid Gangbo and Robert~J McCann.
\newblock The geometry of optimal transportation.
\newblock \emph{Acta Mathematica}, 177\penalty0 (2):\penalty0 113--161, 1996.

\bibitem[Goldfeld et~al.(2022)Goldfeld, Kato, Rioux, and
  Sadhu]{goldfeld2022limit}
Ziv Goldfeld, Kengo Kato, Gabriel Rioux, and Ritwik Sadhu.
\newblock Limit theorems for entropic optimal transport maps and the sinkhorn
  divergence.
\newblock \emph{arXiv preprint arXiv:2207.08683}, 2022.

\bibitem[Golub and Van~Loan(2013)]{golub2013matrix}
Gene~H Golub and Charles~F Van~Loan.
\newblock \emph{Matrix computations}.
\newblock JHU press, 2013.

\bibitem[Grange et~al.(2023)Grange, Al-Jarrah, Baptista, Taghvaei, Georgiou,
  Phillips, and Tannenbaum]{grange2023computational}
Daniel Grange, Mohammad Al-Jarrah, Ricardo Baptista, Amirhossein Taghvaei,
  Tryphon~T Georgiou, Sean Phillips, and Allen Tannenbaum.
\newblock Computational optimal transport and filtering on riemannian
  manifolds.
\newblock \emph{IEEE Control Systems Letters}, 2023.

\bibitem[Huang et~al.(2021)Huang, Ma, and Lai]{pmlr-v139-huang21e}
Minhui Huang, Shiqian Ma, and Lifeng Lai.
\newblock A riemannian block coordinate descent method for computing the
  projection robust wasserstein distance.
\newblock In \emph{Proceedings of the 38th International Conference on Machine
  Learning}, volume 139 of \emph{Proceedings of Machine Learning Research},
  pages 4446--4455. PMLR, 2021.

\bibitem[H{\"u}tter and Rigollet(2021)]{hutter2021minimax}
Jan-Christian H{\"u}tter and Philippe Rigollet.
\newblock Minimax estimation of smooth optimal transport maps.
\newblock \emph{The Annals of Statistics}, 49\penalty0 (2), 2021.

\bibitem[Korotin et~al.(2019)Korotin, Egiazarian, Asadulaev, Safin, and
  Burnaev]{korotin2019wasserstein}
Alexander Korotin, Vage Egiazarian, Arip Asadulaev, Alexander Safin, and Evgeny
  Burnaev.
\newblock Wasserstein-2 generative networks.
\newblock 2019.

\bibitem[Korotin et~al.(2021)Korotin, Li, Genevay, Solomon, Filippov, and
  Burnaev]{korotin2021neural}
Alexander Korotin, Lingxiao Li, Aude Genevay, Justin Solomon, Alexander
  Filippov, and Evgeny Burnaev.
\newblock {Do Neural Optimal Transport Solvers Work? A Continuous Wasserstein-2
  Benchmark}.
\newblock 2021.

\bibitem[Lee and Li(2012)]{lee2012new}
Paul~WY Lee and Jiayong Li.
\newblock New examples satisfying ma--trudinger--wang conditions.
\newblock \emph{SIAM Journal on Mathematical Analysis}, 44\penalty0
  (1):\penalty0 61--73, 2012.

\bibitem[Liero et~al.(2018)Liero, Mielke, and Savar{\'e}]{liero2018optimal}
Matthias Liero, Alexander Mielke, and Giuseppe Savar{\'e}.
\newblock Optimal entropy-transport problems and a new hellinger--kantorovich
  distance between positive measures.
\newblock \emph{Inventiones mathematicae}, 211\penalty0 (3):\penalty0
  969--1117, 2018.

\bibitem[Lin et~al.(2020)Lin, Fan, Ho, Cuturi, and Jordan]{lin2020projection}
Tianyi Lin, Chenyou Fan, Nhat Ho, Marco Cuturi, and Michael Jordan.
\newblock Projection robust wasserstein distance and riemannian optimization.
\newblock \emph{Advances in neural information processing systems},
  33:\penalty0 9383--9397, 2020.

\bibitem[Lin et~al.(2021)Lin, Zheng, Chen, Cuturi, and
  Jordan]{lin2021projection}
Tianyi Lin, Zeyu Zheng, Elynn Chen, Marco Cuturi, and Michael~I Jordan.
\newblock On projection robust optimal transport: Sample complexity and model
  misspecification.
\newblock In \emph{International Conference on Artificial Intelligence and
  Statistics}, pages 262--270. PMLR, 2021.

\bibitem[Liu(2022)]{liu2022rectified}
Qiang Liu.
\newblock Rectified flow: A marginal preserving approach to optimal transport.
\newblock \emph{arXiv preprint arXiv:2209.14577}, 2022.

\bibitem[Ma et~al.(2005)Ma, Trudinger, and Wang]{ma2005regularity}
Xi-Nan Ma, Neil~S Trudinger, and Xu-Jia Wang.
\newblock Regularity of potential functions of the optimal transportation
  problem.
\newblock \emph{Archive for rational mechanics and analysis}, 177:\penalty0
  151--183, 2005.

\bibitem[Makkuva et~al.(2020)Makkuva, Taghvaei, Oh, and
  Lee]{makkuva2020optimal}
Ashok Makkuva, Amirhossein Taghvaei, Sewoong Oh, and Jason Lee.
\newblock Optimal transport mapping via input convex neural networks.
\newblock volume~37, 2020.

\bibitem[Manole et~al.(2021)Manole, Balakrishnan, Niles-Weed, and
  Wasserman]{manole2021plugin}
Tudor Manole, Sivaraman Balakrishnan, Jonathan Niles-Weed, and Larry Wasserman.
\newblock Plugin estimation of smooth optimal transport maps.
\newblock \emph{arXiv preprint arXiv:2107.12364}, 2021.

\bibitem[Monge(1781)]{Monge1781}
Gaspard Monge.
\newblock M{\'e}moire sur la th{\'e}orie des d{\'e}blais et des remblais.
\newblock \emph{Histoire de l'Acad{\'e}mie Royale des Sciences}, 1781.

\bibitem[Muzellec et~al.(2021)Muzellec, Vacher, Bach, Vialard, and
  Rudi]{muzellec2021near}
Boris Muzellec, Adrien Vacher, Francis Bach, Fran{\c{c}}ois-Xavier Vialard, and
  Alessandro Rudi.
\newblock Near-optimal estimation of smooth transport maps with kernel
  sums-of-squares.
\newblock \emph{arXiv preprint arXiv:2112.01907}, 2021.

\bibitem[Niles-Weed and Rigollet(2022)]{niles2022estimation}
Jonathan Niles-Weed and Philippe Rigollet.
\newblock Estimation of wasserstein distances in the spiked transport model.
\newblock \emph{Bernoulli}, 28\penalty0 (4):\penalty0 2663--2688, 2022.

\bibitem[Parikh et~al.(2014)Parikh, Boyd, et~al.]{parikh2014proximal}
Neal Parikh, Stephen Boyd, et~al.
\newblock Proximal algorithms.
\newblock \emph{Foundations and trends{\textregistered} in Optimization},
  1\penalty0 (3):\penalty0 127--239, 2014.

\bibitem[Paty and Cuturi(2019)]{paty2019subspace}
Fran{\c{c}}ois-Pierre Paty and Marco Cuturi.
\newblock Subspace robust wasserstein distances.
\newblock \emph{arXiv preprint arXiv:1901.08949}, 2019.

\bibitem[Peyré and Cuturi(2019)]{Peyre2019computational}
Gabriel Peyré and Marco Cuturi.
\newblock Computational optimal transport.
\newblock \emph{Foundations and Trends in Machine Learning}, 11\penalty0 (5-6),
  2019.
\newblock ISSN 1935-8245.

\bibitem[Pooladian and Niles-Weed(2021)]{pooladian2021entropic}
Aram-Alexandre Pooladian and Jonathan Niles-Weed.
\newblock Entropic estimation of optimal transport maps.
\newblock \emph{arXiv preprint arXiv:2109.12004}, 2021.

\bibitem[Pooladian et~al.(2023{\natexlab{a}})Pooladian, Divol, and
  Niles-Weed]{pooladian2023minimax}
Aram-Alexandre Pooladian, Vincent Divol, and Jonathan Niles-Weed.
\newblock Minimax estimation of discontinuous optimal transport maps: The
  semi-discrete case.
\newblock \emph{arXiv preprint arXiv:2301.11302}, 2023{\natexlab{a}}.

\bibitem[Pooladian et~al.(2023{\natexlab{b}})Pooladian, Domingo-Enrich, Chen,
  and Amos]{pooladian2023neural}
Aram-Alexandre Pooladian, Carles Domingo-Enrich, Ricky~TQ Chen, and Brandon
  Amos.
\newblock Neural optimal transport with {L}agrangian costs.
\newblock In \emph{ICML Workshop on New Frontiers in Learning, Control, and
  Dynamical Systems}, 2023{\natexlab{b}}.

\bibitem[Rigollet and Stromme(2022)]{rigollet2022sample}
Philippe Rigollet and Austin~J Stromme.
\newblock On the sample complexity of entropic optimal transport.
\newblock \emph{arXiv preprint arXiv:2206.13472}, 2022.

\bibitem[Rockafellar(1976)]{rockafellar1976monotone}
R~Tyrrell Rockafellar.
\newblock Monotone operators and the proximal point algorithm.
\newblock \emph{SIAM journal on control and optimization}, 14\penalty0
  (5):\penalty0 877--898, 1976.

\bibitem[Santambrogio(2015)]{santambrogio2015optimal}
Filippo Santambrogio.
\newblock \emph{Optimal transport for applied mathematicians}.
\newblock Springer, 2015.

\bibitem[Schiebinger et~al.(2019)Schiebinger, Shu, Tabaka, Cleary, Subramanian,
  Solomon, Gould, Liu, Lin, Berube, et~al.]{schiebinger2019optimal}
Geoffrey Schiebinger, Jian Shu, Marcin Tabaka, Brian Cleary, Vidya Subramanian,
  Aryeh Solomon, Joshua Gould, Siyan Liu, Stacie Lin, Peter Berube, et~al.
\newblock Optimal-transport analysis of single-cell gene expression identifies
  developmental trajectories in reprogramming.
\newblock \emph{Cell}, 176\penalty0 (4):\penalty0 928--943, 2019.

\bibitem[Sinkhorn(1964)]{Sinkhorn64}
Richard Sinkhorn.
\newblock A relationship between arbitrary positive matrices and doubly
  stochastic matrices.
\newblock \emph{Ann. Math. Statist.}, 35:\penalty0 876--879, 1964.

\bibitem[Srivatsan et~al.(2020)Srivatsan, McFaline-Figueroa, Ramani, Saunders,
  Cao, Packer, Pliner, Jackson, Daza, Christiansen,
  et~al.]{srivatsan2020massively}
Sanjay~R Srivatsan, Jos{\'e}~L McFaline-Figueroa, Vijay Ramani, Lauren
  Saunders, Junyue Cao, Jonathan Packer, Hannah~A Pliner, Dana~L Jackson,
  Riza~M Daza, Lena Christiansen, et~al.
\newblock Massively multiplex chemical transcriptomics at single-cell
  resolution.
\newblock \emph{Science}, 2020.

\bibitem[Tong et~al.(2020)Tong, Huang, Wolf, Van~Dijk, and
  Krishnaswamy]{pmlr-v119-tong20a}
Alexander Tong, Jessie Huang, Guy Wolf, David Van~Dijk, and Smita Krishnaswamy.
\newblock {T}rajectory{N}et: A dynamic optimal transport network for modeling
  cellular dynamics.
\newblock In \emph{Proceedings of the 37th International Conference on Machine
  Learning}, volume 119 of \emph{Proceedings of Machine Learning Research},
  pages 9526--9536. PMLR, 2020.

\bibitem[Villani et~al.(2009)]{villani2009optimal}
C{\'e}dric Villani et~al.
\newblock \emph{Optimal transport: old and new}, volume 338.
\newblock Springer, 2009.

\bibitem[Yang and Uhler(2018)]{yang2018scalable}
Karren~D Yang and Caroline Uhler.
\newblock Scalable unbalanced optimal transport using generative adversarial
  networks.
\newblock \emph{arXiv preprint arXiv:1810.11447}, 2018.

\bibitem[Zou and Hastie(2005)]{zou2005regularization}
Hui Zou and Trevor Hastie.
\newblock Regularization and variable selection via the elastic net.
\newblock \emph{Journal of the Royal Statistical Society Series B: Statistical
  Methodology}, 67\penalty0 (2):\penalty0 301--320, 2005.

\end{thebibliography}
\bibliographystyle{plainnat}

\appendix

\section{More on Subspace Elastic Costs}
\label{sec:appendix_subspace_elastic_costs}
Recall that for a rank-$p$ matrix $A\in\R^{p\times d}$, $p\leq d$, the projection matrix that maps it to its orthogonal is $A^\perp = I-A^T(AA^T)^{-1}A$. When $A$ lies in the Stiefel manifold (i.e. $AA^T =I$), we have the simplification $A^\perp = I-A^T A$. This results in the Pythagorean identity $\|\bz\|^2 = \|A^\perp \bz\|^2+\|A\bz\|^2$, as intended. In order to promote displacements that happen \textit{within} the span of $A$, we must set a regularizer that penalizes the presence of $\bz$ within its \emph{complement}:
$$\tau_{A^\perp}(\bz):= \tfrac{1}{2}\|A^\perp \bz\|^2_2 = \tfrac12 \bz^T(A^\perp)^T A^\perp\bz=\tfrac{1}{2}\bz^T(I_d - A^T(AA^T)^{-1}A)\bz.$$
Since $\tau_{A^\perp}$ is a quadratic form, its proximal operator can be obtained by solving a linear system~\cite[\S6.1.1]{parikh2014proximal}; developing and using the matrix inversion lemma results in 
\begin{equation}
\prox_{\gamma\tau_{{A^\perp}}}(\bz) = \left(I_d + \gamma (A^\perp)^T A^\perp\right)^{-1}\bz = 
\tfrac{1}{1+\gamma}(I+\gamma A^T(AA^T)^{-1}A)\,\bz .
\end{equation}

To summarize, given an orthogonal sub-basis $A$ of $p$ vectors (each of size $d$), promoting that a vector $\bz$ lies in its orthogonal can be achieved by regularizing its norm in the orthogonal of $A$. That norm has a proximal operator that can be computed either by
\begin{enumerate}[leftmargin=.3cm,itemsep=.0cm,topsep=0cm,parsep=2pt]
\item Parameterizing $A$ \textit{implicitly}, through an \textit{explicit} parameterization of an orthonormal basis $B$ for $A^\perp$, as a matrix directly specified in the $(d-p)\times p$ Stiefel manifold. This can alleviate computations to obtain a closed form for its proximal operator:
$$\prox_{\gamma\tau_{A^\perp}}(\bz) = \prox_{\gamma\tau_{B}}(\bz)=\bz - B^T\left(B\bz - \frac{1}{1+\gamma} B\bz\right) = \left(I_d - \frac{\gamma}{1+\gamma}B^T B\right)\bz,$$
but requires storing $B$, a $(d-p)\times d$ orthogonal matrix, which is cumbersome when $p\ll d$.
\item Parameterizing $A$ \textit{explicitly}, either as a full-rank $p\times d$ matrix, or more simply a $p\times d$ orthogonal matrix, to recover the suitable proximal operator for $\tau_{A^\perp}$, by either 
\begin{enumerate}[leftmargin=.3cm,itemsep=.0cm,topsep=0cm,parsep=2pt]
\item Falling back on the right-most expression in~\eqref{eq:prox_orth} in the linear solve, which can be handled using sparse conjugate gradient solvers, since the application of the right-most linear operator has complexity $(p+1)\times d$ and is positive definite, in addition to the linear solve of complexity $O(p^3)$. This simplifies when $A$ is orthogonal, $A\in\mathcal{S}_{p,d}$ since in that case,
\begin{equation}
\prox_{\gamma\tau_{{A^\perp}}}(\bz) = \frac{1}{1+\gamma}\left(I_d + \gamma A^T A\right)\bz.
\end{equation}

\item Alternatively, compute a matrix in the $(d-p)\times p$ Stiefel manifold that spans the same linear space as, through the Gram-Schmidt process~\citep[p.254]{golub2013matrix} of the $d\times d$ matrix $A^\perp$ or rank $d-p$, $B:=\textrm{Gram-Schmidt}(A^\perp)$, to fall back on the expression above.
\end{enumerate}
\end{enumerate}

\section{Proofs from \Cref{sec: statistical_complexity}}\label{app: statistical_proofs}

To perform this analysis, we rely on the following characterization of optimal maps for subspace structured costs, which reveals a close connection with optimal maps for the standard $\ell_2^2$ cost. 

\begin{prop}\label{prop: opt_characterization}
    Let $T^\star$ be the optimal map between $\mu$ and $\nu$ for the cost $h = \tfrac 12 \ell_2^2 + \gamma \tau_{A^\perp}.$
    Denote by $W$ the linear map $\bx \mapsto ((1+\gamma) I - \gamma A^T A)^{1/2} \bx$.
    Then $W \circ T^\star \circ W^{-1}$ is the Brenier map (i.e., $\ell_2^2$ optimal map) between $W_\sharp \mu$ and $W_\sharp \nu$.
    Equivalently, $T^\star$ is $h$-optimal if and only if it can be written 
    \begin{equation}\label{eq: pre_cond_map}
        T^\star = W^{-1} \circ \tilde T \circ W,
    \end{equation}
    where $\tilde T$ is the gradient of a convex function.
\end{prop}

\begin{proof}[Proof of \Cref{prop: opt_characterization}]
The cost $h = \frac 12 \ell_2^2 + \gamma \tau_{A^\perp}$ can be written as
\begin{align*}
    \tfrac{1}{2}\left[z^\top (I + \gamma (A^\perp)^\top A^\perp)z\right] = \tfrac12 \|Wz\|^2\,.
\end{align*}
The optimal transport problem we consider is therefore equivalent to minimizing
\begin{align}\label{eq: pre_cond_problem}
    \min_{\pi\in\Gamma(\mu,\nu)} \int \tfrac12\|Wx-Wy\|^2{\rm{d}}\pi(x,y) = \min_{\pi \in \Gamma(W_\sharp \mu,W_\sharp \nu)} \int \tfrac12\|x'-y'\|^2 {\rm{d}}\pi(x',y')\,.
\end{align}
Brenier's theorem implies that the solution to the latter problem is given by the gradient of a convex function, and that this property uniquely characterizes the optimal map.
Writing this function as $\tilde T$, we obtain that the optimal coupling between $W_\sharp \mu$ and $W_\sharp \nu$ is given by $y' = \tilde T(x')$, which implies that the optimal $h$-coupling between $\mu$ and $\nu$ is given by $T^\star = W^{-1} \circ \tilde T \circ W$, as desired.

\end{proof}

The proof of \Cref{thm: stat_complexity_main} requires the following two lemmas.
\begin{lem}\label{lem: MBO_barycenter}
    For costs of the form $h(z) = \tfrac12 z^\top B z $ where $B$ is positive definite, the MBO estimator between two measures $\mu$ and $\nu$ can be written as the barycentric projection of the corresponding optimal entropic coupling.
\end{lem}
\begin{proof}
    Note that $h^*(w) = \tfrac12 w^\top B^{-1} w$, and thus $\nabla h^*(w) = B^{-1}w$. Let $(f_\eps,g_\eps)$ denote the optimal entropic potentials for this cost, with corresponding coupling $\pi_\eps$. Borrowing computations from \Cref{app: appendix_displacement}, we know that
    \begin{align*}
        \nabla f_\eps(x) = \int B(x-y) {\rm{d}} \pi_\eps^x(y) = Bx - B \int y {\rm{d}} \pi_\eps^x(y)\,,
    \end{align*}
    where $\pi_\eps^x(y)$ is the conditional entropic coupling (given $x$). The proof concludes by taking the expression of the MBO estimator and expanding:
    \begin{align*}
        T_\eps(x) = x - (\nabla h^*) \circ (\nabla f_\eps(x)) = x - B^{-1}\left(Bx - B \int y {\rm{d}} \pi_\eps^x(y) \right) = \int y {\rm{d}} \pi_\eps^x(y)\,,
    \end{align*}
    which is the definition of the barycentric projection of $\pi_\eps$ for a given $x$.
\end{proof}

\begin{lem}[Pre-conditioning of MBO]\label{lem: mbo_precond}
    Let $T_\eps$ be the MBO estimator between $\mu$ and $\nu$ for the cost $h = \tfrac12 \ell^2 + \gamma \tau_{A^\perp}$. Let $W$ be denoted as in \Cref{prop: opt_characterization}. Then the MBO estimator is written as 
    \begin{align}\label{eq: mbo_precond}
        T_\eps = W^{-1} \circ \tilde{T}_\eps \circ W\,,
    \end{align}
    where $\tilde{T}_\eps$ is the barycentric projection between $W_\sharp \mu$ and $W_\sharp \nu$.
\end{lem}
\begin{proof}
The proof here is similar to \Cref{prop: opt_characterization}, which we outline again for completeness. As before, we are interested in solutions to the optimization problem
\begin{align*}
    \min_{\pi \in \Gamma(\mu,\nu)} \int \tfrac12\|Wx - Wy\|^2 {\rm{d}}\pi(x,y) + \eps \text{KL}(\pi\|\mu \otimes \nu)\,,
\end{align*}
with optimal coupling $\pi_\eps^\star$. Performing a change of variables $\pi' = (W \otimes W)_\sharp \pi$, we have
\begin{align*}
        \min_{\pi' \in \Gamma(\mu',\nu')} \int \tfrac12\|x - y\|^2 {\rm{d}}\pi'(x,y) + \eps \text{KL}(\pi'\|\mu' \otimes \nu')\,,
\end{align*}
where $\mu' := W_\sharp \mu$ (and similarly for $\nu'$), where now the optimizer reads $(\pi_\eps')^\star$. The two optimal plans are related as
\begin{align*}
    \pi^\star = (W^{-1} \otimes W^{-1})_\sharp (\pi_\eps')^\star\,.
\end{align*}
It was established in \Cref{lem: MBO_barycenter} that the MBO estimator $T^\star_\eps$ is given by the barycentric projection
\begin{align*}
    T_\eps^\star(x) = \mathbb{E}_{\pi_\eps^\star}[Y|X=x]\,.
\end{align*}
Performing the change of variables $Y' = WY$ and $X' = WX$, we can re-write this as a function of $\pi'_\eps$ instead:
\begin{align*}
    T^\star_\eps(x) &= \mathbb{E}_{\pi_\eps^\star}[Y|X=x] \\
    &= \mathbb{E}_{(\pi_\eps')^\star}[W^{-1}Y' | W^{-1}X' = x] \\
    &= W^{-1}\mathbb{E}_{(\pi_\eps')^\star}[Y' | X' = Wx] \\
    &= W^{-1} \tilde{T}_\eps(Wx)\,,
\end{align*}
where we identify $\tilde{T}_\eps(\cdot) := \mathbb{E}_{(\pi_\eps')^\star}[Y' | X' = \cdot]$; this completes the proof.
\end{proof}
We are now ready to present the main proof.
\begin{proof}[Proof of \Cref{thm: stat_complexity_main}] 
Let $T_{\eps,n}$ denote the MBO estimator between samples from $\mu$ and $\nu$, and let $\tilde{T}_{\eps,n}$ denote the entropic map estimator from samples $\mu' := W_\sharp \mu$ and $\nu' := W_\sharp \nu$, where $W$ has spectrum $0 < \lambda_{\min}(W) \leq \lambda_{\max}(W) < +\infty$, where we have access to $W$ since $A$ is known.

Our goal is to establish upper bounds on 
\begin{align*}
    \| T_{\eps,n} - T^\star\|^2_{L^2(\mu)} = \| W^{-1} \circ (\tilde{T}_{\eps,n} \circ W - \tilde{T} \circ W)\|^2_{L^2(\mu)}\,.
\end{align*}
Paying for constants that scale like $\lambda_{\max}(W^{-1})$, we have the bound
\begin{align*}
    \| T_{\eps,n} - T^\star\|^2_{L^2(\mu)} \lesssim_W \|\tilde{T}_{\eps,n} - \tilde{T}\|^2_{L^2(\mu')}\,,
\end{align*}
where we can now directly use the rates of convergence from \cite[Theorem 3]{pooladian2021entropic}, as $\mu'$ satisfies our regularity assumptions under the conditions we have imposed on $W$. this completes the proof.
\end{proof}

\section{Gradient of Elastic Cost Loss}\label{sec:explainloss}
The gradient of the loss $\mathcal{L}$ in \ref{defn:loss} can be recovered through a simple aggregation of weighted gradients
\begin{equation}
\label{eq:chain_rule_simple}
    \nabla \mathcal{L}(\theta) = \sum_{ij} \left[R(\theta)\right]_{ij} \nabla_\theta\left[P^\star(\theta)\right]_{ij} + \left[P^\star(\theta)\right]_{ij} \nabla_\theta \left[R(\theta)\right]_{ij}. 
\end{equation}
To write this formula in a more compact way, it is sufficient to notice that, adopting the convention that $\theta$ be a parameter in $\mathbb{R}^q$, and introducing an arbitrary vector $\omega\in\mathbb{R}^q$,

$$
\begin{aligned}
    \langle\nabla \mathcal{L}(\theta),\omega\rangle &= \sum_{ij} \left[R(\theta)\right]_{ij} \langle \nabla_\theta\left[P^\star(\theta)\right]_{ij},\omega\rangle + \left[P^\star(\theta)\right]_{ij} \langle\nabla_\theta \left[R(\theta)\right]_{ij},\omega\rangle\\
    &= \langle R(\theta), \left[\langle\nabla_\theta\left[P^\star(\theta)\right]_{ij},\omega\rangle\right]_{ij} \rangle + \langle P^\star(\theta), \left[\langle\nabla_\theta\left[R(\theta)\right]_{ij},\omega\rangle\right]_{ij} \rangle\,. 
\end{aligned}
$$
The products of all coordinate wise gradients with $\omega$ is equivalent to the application of the Jacobians of $R$ and $P^\star$. 
We write $J_\theta P^\star$ and $J_\theta R$ for these Jacobians, both being maps taking $\theta$ as input, and outputting a linear map $J_\theta R: \mathbb{R}^{q}\rightarrow \mathbb{R}^{n\times m}$, i.e. $J_\theta R(\theta)$ is a $n\times m$ matrix. As a consequence one has
$$
\begin{aligned}
    \langle\nabla \mathcal{L}(\theta),\omega\rangle &= \langle R(\theta), J_\theta P^\star(\theta)\,\omega \rangle + \langle P^\star(\theta), J_\theta R(\theta)\,\omega \rangle 
\end{aligned}
$$
because these maps are linear, one also has
$$
\begin{aligned}
    \langle\nabla \mathcal{L}(\theta),\omega\rangle &= \langle J_\theta^T P^\star(\theta) R(\theta), \omega \rangle + \langle J_\theta^T R(\theta) P^\star(\theta), \omega \rangle \\
    &= \langle J_\theta^T P^\star(\theta) R(\theta) + J_\theta^T R(\theta) P^\star(\theta), \omega \rangle
\end{aligned}
$$
which gives the identification given in the main text.

\section{Additional Details on Experiments}\label{app:exps}

In \S~\ref{subsec:learningsub}, and unlike Figure~\ref{fig:mse}, we do not choose a predefined value for $\gamma^\star$, but instead select it with the following procedure: we start with a small value for $\gamma_0=0.1$, and increase it gradually, until a certain desirable criterion on these displacements goes above a threshold. To measure this, we first compute the (paired) matrix of displacements on a given sample,
$$D = [T_{g}^h(\bx_i) - \bx_i]_i \in\R^{n\times d}$$

We then consider ratio of singular values on $p^*$ subspace (to select $\gamma$ for $\|A^\perp\cdot\|^2_2$), writing $\sigma$ for the vector of singular values of $D$, ranked in decreasing order, to compute \begin{equation}\label{eq:svr}\textrm{sv-ratio}(\gamma) =  \sum_{i=1}^p \sigma_i / \sum_i \sigma_i \in [0,1]\,.
\end{equation}

\end{document}